\documentclass[12pt, reqno]{article}

\usepackage[hypertexnames=false]{hyperref}
\hypersetup{
    setpagesize=false,
    bookmarksnumbered=true,%
    bookmarksopen=true,%
    colorlinks=true,%
    linkcolor=blue,
    citecolor=red,
}

\usepackage[utf8]{inputenc} 
\usepackage[T1]{fontenc}    
\usepackage{hyperref}       
\usepackage{url}            
\usepackage{booktabs}       

\usepackage{amsfonts}       
\usepackage{nicefrac}       
\usepackage{microtype}      

\usepackage{amsthm}
\usepackage{amsmath}
\usepackage{amssymb} 
\usepackage{mathtools}

\usepackage{euscript}
\usepackage{enumitem} 

\usepackage{graphicx}
\usepackage[capitalize]{cleveref}
\usepackage{autonum} 

\renewcommand{\ref}{\cref}

\theoremstyle{plain}
\newtheorem{thm}{Theorem}[section]
\newtheorem{prop}[thm]{Proposition}
\newtheorem{lemma}[thm]{Lemma}
\newtheorem{cor}[thm]{Corollary}

\theoremstyle{definition}
\newtheorem{defn}[thm]{Definition}
\newtheorem{rmk}[thm]{Remark}
\newtheorem{example}[thm]{Example}

\newtheorem{note}[thm]{Notation}

\numberwithin{equation}{section}

\newcommand{\N}{\mathbb N}

\newcommand{\R}{\mathbb R}
\newcommand{\C}{\mathbb C}

\newcommand{\E}{\mathbb E}  
\renewcommand{\O}[1]{\mathbf O_{#1}} 

\newcommand{\mc}[1]{\mathcal{#1}}
\newcommand{\mf}[1]{\mathfrak{#1}}

\newcommand{\eps}{\varepsilon}

\DeclareMathOperator{\tr}{\mathrm{tr}}
\DeclareMathOperator{\Tr}{\mathrm{Tr}}

\DeclareMathOperator{\rank}{\mathrm{rank}}
\DeclareMathOperator{\diag}{\mathrm{diag}}

\DeclareMathOperator{\Normal}{\mathcal{N}}

\newcommand{\To}{\Rightarrow}
\newcommand{\pp}[2]{\frac{\partial #1}{\partial #2}}
\newcommand{\abs}[1]{|#1|}
\newcommand{\pre}{h}
\newcommand{\post}{x}
\newcommand{\act}{\varphi}

\newcommand{\Hmltn}{H}

\newcommand{\embedding}[1]{ \begin{pmatrix} #1 & 0 \\ 0 & 0 \end{pmatrix}}
\newcommand{\entries}{\sim^\mathrm{entries}}

\usepackage{tikz}
\usepackage{systeme,mathtools}
\usetikzlibrary{positioning,arrows.meta,quotes,automata}
\usetikzlibrary{shapes,snakes}
\usetikzlibrary{bayesnet}

\tikzset{>=latex}
\tikzstyle{plate caption} = [caption, node distance=0, inner sep=0pt,
below left=5pt and 0pt of #1.south]

\tikzset{circ/.style={
    circle,
    draw=black,
    fill=white,
    inner sep=0pt,
    minimum size=1cm,
    }
}

\tikzset{box/.style={
    rectangle,
    draw=black,
    fill=white,
    inner sep=0pt,
    minimum size=1cm,
    }
}

\tikzset{empty style/.style={
    circle,
    draw=white,
    fill=white,
    inner sep=0pt,
    minimum size=0.6cm
    }
}

\tikzset{mark style/.style={
    mark size=1pt,
    color=black
    }
}    
\tikzset{text style/.style={
    text width=0.5cm
    }
}

\usepackage[symbol]{footmisc}

\usepackage{comment}
\usepackage{color}

\title{ Asymptotic Freeness of Layerwise Jacobians Caused by Invariance 
of Multilayer Perceptron: \\ The Haar Orthogonal Case}

\author{
Beno\^\i{}t Collins$^*$, Tomohiro Hayase$^{**}$
}

\date{\vspace{-5ex}} 
\begin{document}

\footnotetext{$^*$ Kyoto University}
\footnotetext{$^{**}$ Fujitsu Laboratories, hayase.tomohiro@fujitsu.com, hayafluss@gmail.com}
\footnotetext{The authors are equally contributed.}

\maketitle

\begin{abstract}
Free Probability Theory (FPT) provides 
rich knowledge for handling mathematical difficulties caused by random matrices appearing in research related to deep neural networks (DNNs), such as the dynamical isometry, Fisher information matrix, and training dynamics.  
FPT suits these researches because the DNN's parameter-Jacobian and input-Jacobian are polynomials of layerwise Jacobians.
However, the critical assumption of asymptotic freeness of the layerwise Jacobian has not been proven mathematically so far. 
The asymptotic freeness assumption plays a  fundamental role 
when propagating
spectral distributions through the layers.
Haar distributed orthogonal matrices are essential for achieving dynamical
isometry. In this work, we prove asymptotic freeness of layerwise Jacobians of
multilayer perceptron (MLP) in this case.
A key of the proof is an invariance of the MLP. Considering the orthogonal matrices that fix the hidden units in each layer, we replace each layer's parameter matrix with itself multiplied by the orthogonal matrix, and then the MLP does not change. Furthermore, if the original weights are Haar orthogonal, the Jacobian is also unchanged by this replacement. Lastly, we can replace each weight with a Haar orthogonal random matrix independent of the Jacobian of the activation function using this key fact. 
\end{abstract}
\tableofcontents
\section{Introduction}
Free Probability Theory (FPT) provides essential insight when handling mathematical difficulties caused by random matrices that appear in deep neural networks (DNNs)  \cite{Pennington2018emergence, Hanin2019products, Hayase2020spectrum}. 
The DNNs have been successfully used to achieve empirically high performance in various machine learning tasks \cite{LeCun2015DeepLearning, Goodfellow2016DeepLearning}.
However, their understanding at a theoretical level is limited, and their success relies heavily on heuristic search settings such as architecture and hyperparameters.
To understand and improve the training of DNNs, researchers have developed several theories to investigate, for example, the vanishing/exploding gradient problem \cite{Schoenholz2016DeepPropagation}, the shape of the loss landscape \cite{Pennington2018spectrum, Karakida2019universal}, and the global convergence of training and generalization \cite{Jacot2018NeuralNetworks}.
The nonlinearity of activation functions, the depth of DNN, and
the lack of commutation of random matrices result in significant mathematical
challenges. In this respect, FPT, invented by Voiculescu \cite{voiculescu1985symmetries, Voiculescu1987multiplication, Voiculescu1991limit},  is well suited for this kind of analysis.

FPT essentially appears in the analysis of the dynamical isometry  \cite{Pennington2017resurrecting, Pennington2018emergence}.
It is well known that reducing the training error in very deep models is difficult without carefully preventing the gradient's vanishing/exploding. 
Naive settings (i.e., activation function and initialization) cause vanishing/exploding gradients, as long as the network is relatively deep.
The dynamical isometry \cite{Saxe2014exact,Pennington2018emergence} was proposed to solve this problem. The dynamical isometry can facilitate training by setting the input-output Jacobian's singular values to be one, where the input-output Jacobian is the Jacobian matrix of the DNN at a given input.  
Experiments have shown that with initial values and models satisfying dynamical isometry, very deep models can be trained without gradient vanishing/exploding;
\cite{Pennington2018emergence, Xiao2018dynamical, Sokol2018information} have found that DNNs achieve approximately dynamical isometry over random orthogonal weights, but they do not do so over random Gaussian weights. 
For the sake of the prospect of the theory, let $J$ be the Jacobian of the multilayer perceptron (MLP),
which is the fundamental model of DNNs.
The Jacobian $J$ is given by the product of layerwise Jacobians:
\begin{align}
    J = D_L W_L \dots  D_1 W_1,
\end{align}
where each $W_\ell$ is $\ell$-th weight matrix,  each $D_\ell$ is Jacobian of $\ell$-th activation function, and $L$ is the number of layers.
Under an assumption of asymptotic freeness, the limit spectral distribution is given by \cite{Pennington2018emergence}.

To examine the training dynamics of MLP achieving the dynamical isometry, \cite{Hayase2020spectrum} introduced a spectral analysis of the Fisher information matrix per sample of MLP.
The Fisher information matrix (FIM) has been a fundamental quantity for such theoretical understandings. The FIM describes the local metric of the loss surface concerning the KL-divergence function \cite{Amari2016}.  The neural tangent kernel  \cite{Jacot2018NeuralNetworks}, which has the same eigenvalue spectrum except for trivial zero as FIM, also describes the learning dynamics of DNNs when the dimension of the last layer is relatively smaller than the hidden layer. 
In particular, the FIM's eigenvalue spectrum describes the efficiency of optimization methods. 
For instance, the maximum eigenvalue determines an appropriate size of the learning rate of the first-order gradient method for convergence \cite{Cun1991eigenvalues, Karakida2019universal, Wu2019how}.
Despite its importance in neural networks,  the FIM spectrum has been the object of only very little study from a theoretical perspective.
The reason is that it was limited to random matrix theory for shallow networks \cite{Pennington2018spectrum} or mean-field theory for eigenvalue bounds, which may be loose in general \cite{Karakida2019normalization}.  
Thus,  \cite{Hayase2020spectrum} focused on the FIM per sample and found an alternative approach applicable to DNNs.
The FIM per sample is equal to $J_\theta^\top J_\theta$, where $J_\theta$ is the parameter Jacobian.
Also, the eigenvalues of the FIM per sample are equal to the eigenvalues of the $H_L$ defined recursively as follows, except for the trivial zero eigenvalues and normalization:
\begin{align}\label{align:reccurence_matrix}
\Hmltn_{\ell +1 } =   \hat{q}_{\ell}I + W_{\ell+1} D_{\ell}\Hmltn_{\ell}D_{\ell}W_{\ell+1}^\top, \ \ell=1, \dots, L-1,
\end{align}
where $I$ is the identity matrix, and $\hat{q}_\ell$ is the empirical variance of $\ell$-th hidden unit.
Under an asymptotic freeness assumption, \cite{Hayase2020spectrum} gave some limit spectral distributions of $H_L$.

The asymptotic freeness assumptions have a critical role in these researches \cite{Pennington2018emergence, Hayase2020spectrum} to obtain the propagation of spectral distributions through the layers.
However, the proof of the asymptotic freeness was not completed.
In the present work, we prove the asymptotic freeness of layerwise Jacobian of multilayer perceptrons with Haar orthogonal weights.

\subsection{Main Results}
Our results are as follows. Firstly, the following $L+1$ tuple  of families are asymptotically free almost surely (see \cref{thm:main}):
\begin{align}
 (   (W_1, W_1^*),\dots, (W_L,W_L^*) , (D_1, \dots, D_L)  ).
\end{align}
Secondly, for each $\ell=1, \dots, L-1$, the following pair  is almost surely asymptotically free (see \cref{prop:input-jacobian}):
\begin{align}
      W_{\ell+1} J_\ell J_\ell^*  W_{\ell+1} , D_\ell^2.
\end{align}
The asymptotic freeness is at the heart of the spectral analysis of the Jacobian.
Lastly, for each $\ell=1, \dots, L-1$, the following pair is almost surely asymptotically free (see \cref{prop:parameter-jacobian}):
\begin{align}
    H_\ell , D_\ell^2.
\end{align}
The asymptotic freeness of the pair is the key to the analysis of the conditional Fisher information matrix.

The fact that each parameter matrix $W_\ell$ contains elements correlated with the activation's Jacobian matrix $D_\ell$ is a hurdle towards showing asymptotic freeness.
Therefore, among the components of $W_\ell$, we move the elements that appear in $D_\ell$ to the $N$-th row or column. This is achieved by changing the basis of $W_\ell$.
The orthogonal matrix \eqref{align:Y_defn} that defines the change of basis can be realized so that each hidden layer is fixed, and as a result, the MLP does not change. Then, the dependency between $W_\ell$ and $D_\ell$ is only in the $N$-th row or column, so it can be ignored by taking the limit of $N \to \infty$. From this result, we can say that $(W_\ell, W_\ell^\top)$ and $D_\ell$ are asymptotically free for each $\ell$. However, this is still not enough to prove the asymptotical freeness between families $(W_\ell, W_\ell^\top)_{\ell=1, \dots, L}$ and $(D_\ell)_{\ell=1, \dots, L}$. Therefore, we complete the proof of the asymptotic freeness by additionally considering another change of basis \eqref{align:practical_U} that rotates the $N-1 \times N-1$ submatrix of each $W_\ell$  by independent Haar orthogonal matrices.
A key of the desired asymptotic freeness is the invariance of MLP described in \cref{lemma:Ux}.  The invariance follows from a structural property of MLP and an invariance property of Haar orthogonal random matrices. The  invariance of MLP helps us apply the asymptotical freeness of Haar orthogonal random matrices \cite{collins2006integration} to our situation.

\subsection{Related Works}

The asymptotic freeness is weaker than the assumption of the forward-backward independence that research of dynamical isometry assumed  \cite{Pennington2017resurrecting,Pennington2018emergence,Karakida2019universal}. 
Although studies of mean-field theory \cite{Saxe2014exact, LeCun2015DeepLearning,Gilboa2019dynamical}  succeeded in explaining many experimental deep learning results, they use an artificial assumption (gradient independence \cite{Yang2019scaling}), which is not rigorously true.
Asymptotic freeness is weaker than this artificial assumption. 
Our work clarifies that asymptotic free independence is just the right property that is useful and strictly valid for analysis.

Several works prove or treat the asymptotic freeness with Gaussian initialization \cite{Hanin2019products, Yang2019scaling, Yang2020tensor, Pastur2020random}.
However, asymptotic freeness was not proven for the orthogonal initialization.  As dynamical isometry can be achieved under orthogonal initialization but cannot be done under Gaussian initialization \cite{Pennington2018emergence},  proof of the asymptotic freeness in orthogonal initialization is essential.
Since our proof makes crucial use of the properties of Haar distributed random matrices, the proof is clear because we only need to aim to replace the weights with Haar orthogonal, which is independent of the other Jacobians.
While \cite{Hanin2019products} restricting the activation function to ReLU, our proof covers a comprehensive class of activation functions, including smooth functions.

\subsection{Organization of the Paper}
\cref{sec:preliminaries} is devoted to preliminaries. It contains settings of MLP and notations about random matrices, spectral distribution, and free probability theory.
\cref{sec:keys} consists of two keys to prove main results. A key is the invariance of MLP, and the other is to cut off a dimension.
\cref{sec:main} is devoted to proving the main results on the asymptotic freeness. 
In \cref{sec:appl}, we show applications of the asymptotic freeness to spectral analysis of random matrices, which appear in the theory of dynamical isometry and training dynamics of DNNs.
\cref{sec:discussion} is devoted to the discussion and future works.
\section{Preliminaries}\label{sec:preliminaries}
\subsection{Setting of MLP}\label{ssec:setting}
We consider multilayer perceptron settings, as usual in the studies of FIM \cite{Pennington2018spectrum, Karakida2019universal} and dynamical isometry \cite{Saxe2014exact, Pennington2018emergence, Hayase2020spectrum}. 
Fix $L,N \in \N$. We consider an $L$-layer multilayer perceptron as a parametrized map $f=(f_\theta \mid \theta=(W_1, \dots, W_L) )$ with weight matrices $W_1, W_2, \dots, W_L \in M_N(\R)$ as follows.
Firstly, consider functions $\act^1, \dots \act^{L-1}$ on $\R$. Besides, we assume that $\act^\ell$ is continuous and differentiable except for finite points.
Secondly, for a single input  $x \in \R^N$ we set $\post^0=x$. 
In addition, for $\ell=1, \dots, L$, set inductively
\begin{equation}
    \pre^\ell  = W_\ell \post^{\ell-1} + b^\ell, \ \ 
      \post^\ell  = \act^\ell(\pre^\ell ),
    \label{align:post-to-pre}
\end{equation}
where $\act^\ell$ acts on $\R^N$ as the entrywise operation.
Note that we set $b^\ell = 0$ to simplify the analysis.
Write $f_\theta(x) = \post^L$.
Denote by $D_\ell$ the Jacobian of the activation $\act^\ell$ given by
\begin{align}
    D_\ell = \pp{\post^\ell}{\pre^\ell} = \diag( (\act^\ell)^\prime(\pre^\ell_1), \dots, (\act^\ell)^\prime(\pre^{\ell}_N)).
\end{align}
Lastly,  we assume that each $W_\ell$ ($\ell=1, \dots, L$) be independent Haar orthogonal random matrices and further consider the following condition (d1), \dots, (d4) on distributions.
In \cref{fig:graphical}, we visualize the dependency of the random variables.

\begin{figure}[h]
    \centering

    \begin{tikzpicture}
    \node [circ] (x0) at (0,0) { $\boldsymbol{x}^0$};
    \node [circ] (h1) at (1.5,0) { $\boldsymbol{h}^1$};
    \node [box] (w1) at (1.5,2) { $\boldsymbol{W}^1$};
    \node [box] (d1) at (1.5,-2) { $\boldsymbol{D}^1$};

    \node [circ] (x1) at (3,0) { $\boldsymbol{x}^1$};
    \node [circ] (h2) at (4.5,0) { $\boldsymbol{h}^2$};
    \node [box] (w2) at (4.5,2) { $\boldsymbol{W}^2$};
    \node [box] (d2) at (4.5,-2) { $\boldsymbol{D}^2$};
    
    \node [empty style] (x2) at (6,0) {};

    \node [empty style] (hL-1) at (9,0) {};

	\node[mark style] at (7.5,0) {$\cdots$};

    \node [circ] (xL-1) at (10.5,0) { $\boldsymbol{x}^{L-1}$};
    \node [circ] (hL) at (12,0) { $\boldsymbol{h}^L$};
    \node [box] (wL) at (12,2) { $\boldsymbol{W}^L$};
    \node [box] (dL) at (12,-2) { $\boldsymbol{D}^L$};

    \node [circ] (xL) at (13.5,0) { $\boldsymbol{x}^L$};

    \draw[every loop,
          auto=left,
          line width=0.2mm,
          >=latex,
          draw=black,
          fill=black]    
        (x0) edge (h1)
        (w1) edge (h1)
        (h1) edge node {$\diag \circ (\boldsymbol{\act}^1)^\prime$} (d1)
        (h1) edge node {$\boldsymbol{\act}_1$}(x1)

        (x1) edge (h2)
        (w2) edge (h2)
        (h2) edge node {$\diag \circ (\boldsymbol{\act}^2)^\prime$} (d2)
        (h2) edge node {$\boldsymbol{\act}^2$}(x2)

        (hL-1) edge (xL-1)
        (wL) edge (hL)
        (xL-1) edge (hL)
        (hL) edge node {$\diag \circ (\boldsymbol{\act}^L)^\prime$} (dL)
        (hL) edge node {$\boldsymbol{\act}^L$}(xL)
        ;

\end{tikzpicture}
    \caption{A graphical model of random matrices and random vectors drawn by the following rules (i--iii). (i)A node's boundary is drawn as a square or a rectangle if it contains a square random matrix; otherwise, it is drawn as a circle. (ii)For each node, its parent node is a source node of a directed arrow. A node is measurable concerning the $\sigma$-algebra generated by all parent nodes. (iii)The nodes which have no parent node are independent.}
    \label{fig:graphical}
\end{figure}
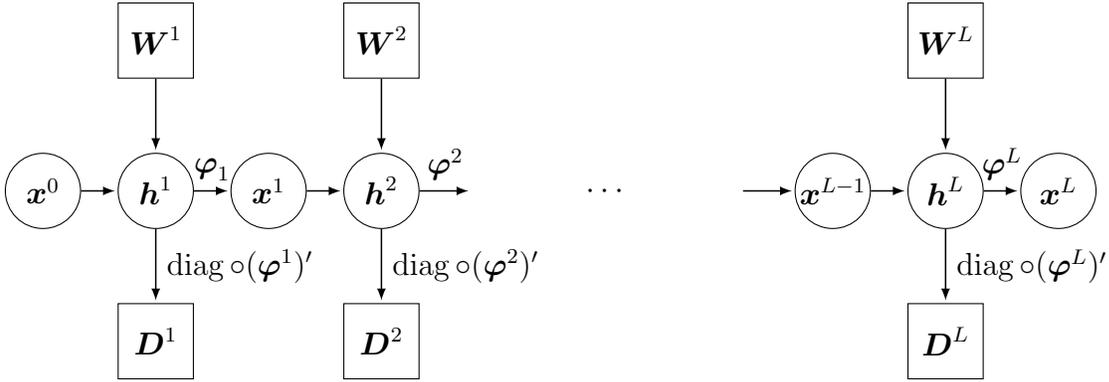

\begin{enumerate}[label=\text{(d\arabic*)}, ref=\text{(d\arabic*)}]
\item  For each $N \in \N$, the input vector $\post^0$ is $\R^N$-valued random variable such that  there is $r > 0$ with
    \begin{align}
     \lim_{N \to \infty}||\post^0||_2/\sqrt{N} = r
    \end{align}
    almost surely.
    \item Each weight matrix $W_\ell$ ($\ell=1, \dots, L$) satisfies
    \begin{align}
         W_\ell = \sigma_{w,\ell} O_\ell,
    \end{align}
    where $O_\ell~(\ell=1,\dots, L)$ are independent  orthogonal matrices distributed with the Haar probability measure and $\sigma_{w,\ell} > 0$.
    \item For fixed $N$, the family 
    \begin{align}
        (\post^0, W_1, \dots, W_L)
    \end{align}  is independent.
\end{enumerate}

Let us define $r_\ell > 0$ and $q_\ell > 0$ by the following recurrence relations:
\begin{align}
    r_0 &=  r,\\
    (r_\ell)^2 &=   \E_{h \sim \Normal(0, q_\ell)} \left[ \act_\ell\left( h \right)^2 \right] \ (l=1, \dots, L),\\
 	q_\ell &= (\sigma_{w,\ell})^2 (r_{\ell-1})^2     \ (l=1, \dots, L).
\end{align}
The inequality $r_\ell < \infty$ holds by the assumption \ref{enum:act_L2} of activation functions.

We further assume that each activation function satisfies the following conditions (a1), \dots, (a5).
\begin{enumerate}[label=\text{(a\arabic*)}, ref=\text{(a\arabic*)}]
    \item It is a continuous function on $\R$ and is not the identically zero function.
    \item For any $q >0$, 
    \begin{align}
        \int_\R \act^\ell(x)^2 \exp(-x^2/q)dx < \infty.
    \end{align}
    \label{enum:act_L2}
    \item It is differentiable almost everywhere concerning Lebesgue measure. We denote by $(\act^\ell)^\prime$ the derivative defined almost everywhere.
    \item The derivative $(\act^\ell)^\prime $ is continuous almost everywhere concerning the Lebesgue measure. \label{enum:act_deriv_conti}
    \item The derivative $(\act^\ell)^\prime $ is bounded. \label{enum:act_deriv_bounded}
\end{enumerate}


\begin{example}[Activation Functions]
The following activation functions are used te{Pennington2017resurrecting, Pennington2018emergence, Hayase2020spectrum} to satisfy the above conditions.
\begin{enumerate}
    \item (Rectified linear unit)
    
    \begin{align}
    \mathrm{ReLU}(x) =
    \begin{cases}
     x ; &x \geq 0,\\
     0 ; &x < 0.
    \end{cases}
     \end{align} 

    \item (Shifted ReLU)
    \begin{align}
    \text{shifted-ReLU}_\alpha(x)
     =
    \begin{cases}
     x ;& x \geq \alpha,\\
     \alpha ;& x < \alpha.
    \end{cases}
     \end{align} 

    \item (Hard hyperbolic tangent)
    \begin{align}
    \text{htanh}(x) 
     =  
    \begin{cases}
     -1 ;&  x \leq -1,\\
     x ;&  -1 < x  <  1,\\
     1 ;&  1 \leq x.
     \end{cases}
     \end{align}

    \item (Hyperbolic tangent)
    \begin{align}
    \tanh(x) = \frac{e^x- e^{-x}}{e^x+e^{-x}}.
    \end{align}

    \item  (Sigmoid function)
    \begin{align}
    \sigma(x) = \frac{1}{e^{-x}+1}.
     \end{align} 

    \item (Smoothed ReLU)
    \begin{align}
    \mathrm{SiLU}(x) = x \sigma(x).
     \end{align} 

    \item (Error function)
     \begin{align}
         \text{erf}(x) = \frac{2}{\sqrt{\pi}}\int^x_0 e^{-t^2}dt.
     \end{align}

\end{enumerate}

\end{example}

\subsection{Basic Notations}
\paragraph{Linear Algebra}
We denote by $M_N(\mathbb{K})$  the algebra of $N \times N$ matrices with entries in a field $\mathbb{K}$. 
Write unnormalized and normalized traces of $A \in M_N(\mathbb{K})$ as follows:
\begin{align}
\Tr(A) &= \sum_{i=1}^N A_{ii},\\
\tr(A) &= \frac{1}{N}\Tr(A).
\end{align}
In this work, a random matrix is a $M_N(\mathbb{\R})$ valued Borel measurable map from a fixed probability space for an $N \in \N$.
We denote by $\O{N}$ the group of $N \times N$ orthogonal matrices. It is well-known that $\O{N}$ is equipped with 
a unique left and right translation invariant probability measure, 
called the Haar probability measure.

\paragraph{Spectral Distribution}
Recall that the  spectral distribution $\mu$ of a linear operator $A$ is a probability distribution $\mu$ on $\R$ such that $\tr(A^m) = \int t^m\mu(dt)$ for any $m \in \N$, where $\tr$ is the normalized trace.
If $A$ is an $N \times N$ symmetric matrix with $N \in \N$, its spectral distribution is given by $N^{-1}\sum_{n=1}^N \delta_{\lambda_n}$, where $\lambda_n (n=1, \dots, N)$ are eigenvalues of $A$, and $\delta_\lambda$ is the discrete probability distribution whose support is $\{\lambda\} \subset \R$.

\paragraph{Joint Distribution of All Entries}
For random matrices $X_1, \dots ,X_L, Y_1, \dots, Y_L$ and random vectors $x_1, \dots, x_L, y_1, \dots, y_L$, we write \begin{align}
    (X_1, \dots, X_L, x_1, \dots, x_L) \entries (Y_1, \dots, Y_L, y_1, \dots, y_L)
\end{align}
if 
the
joint distributions of all entries of corresponding matrices and vectors in the families match.

\subsection{Asymptotic Freeness}

\noindent In this section, we summarize required topics of random matrices and free probability theory.
We start with the following definition.
We omit the definition of a C$^*$-algebra, and for complete details, we refer to \cite{Mingo2017free}.
\begin{defn}\label{defn:ncps}
A  \emph{noncommutative $C^*$-probability space} (NCPS, for short)  is a pair $(\mf{A},\tau ) $ of a unital C$^*$-algebra $\mf{A}$ and a faithful tracial state  $\tau$ on $\mf{A}$, which are defined as follows.
A linear map $\tau$ on $\mf{A}$ is said to be a $\emph{tracial state}$ on $\mf{A}$ if the following four conditions are satisfied.
\begin{enumerate}
    \item $\tau(1) =1$.
    \item $\tau(a^*) = \overline{\tau(a)} \  (a \in \mf{A})$.
    \item $\tau(a^*a) \geq 0 \  (a \in \mf{A})$.
    \item $\tau(ab) = \tau(ba) \ (a,b \in \mf{A})$.
\end{enumerate}
In addition, we say that $\tau$ is \emph{faithful} if $\tau(a^*a)=0$ implies $a=0$.
\end{defn}

For $N \in \N$, the pair of the algebra $M_N(\C)$ of $N \times N$ matrices of complex entries and the normalized trace $\tr$ is an NCPS.
Consider the algebra of $M_N(\R)$ of $N \times N$ matrices of real entries and the normalized trace $\tr$. The pair itself is not an NCPS in the sense of \cref{defn:ncps} since it is not $\C$-linear space.
However,  $M_N(\C)$ contains $M_N(\R)$  and preserves $*$ by setting, 
for $A \in M_N(\R)$: 
\begin{align}
    A^* = A^\top.
\end{align}
Also, the inclusion $M_N(\R) \subset M_N(\C)$ preserves the trace.
Therefore, we consider the joint distributions of matrices in $M_N(\R)$ as that of elements in the NCPS $(M_N(\C), \tr)$.

\begin{defn}(Joint Distribution in NCPS)
Let $a_1, \dots, a_k \in \mf{A}$ and  let $\C \langle X_1, \dots, X_k \rangle$ be the free algebra of non-commutative polynomials on $\C$ generated by $k$ indeterminates $X_1, \dots, X_k$.  Then the \emph{joint distirubtion} of the $k$-tuple $(a_1, \dots, a_k)$ is the linear form $\mu_{a_1, \dots, a_k} : \C \langle X_1, \dots, X_k \rangle \to \C$
defined by
\begin{align}
    \mu_{a_1, \dots, a_k} ( P  ) = \tau (P(a_1, \dots, a_k)),
\end{align}
where $P \in \C \langle X_1, \dots, X_k \rangle$.
\end{defn}

\begin{defn}
    Let $a_1, \dots, a_k \in \mf{A}$.
    Let $A_1(N), \dots, A_k(N)$ $( N \in \N)$ be sequences of $N \times N$ matrices.
    Then we say that they converge in distribution to $(a_1, \dots, a_k)$ if 
    \begin{align}
        \lim_{N \to \infty}\tr \left( P \left( A_1 (N), \dots, A_k(N) \right) \right)   = \tau\left( P \left(a_1, \dots, a_k\right)\right)
    \end{align}
    for any $P \in \C \langle X_1, \dots, X_k \rangle$.
\end{defn}

\begin{defn}(Freeness)
Let $(\mf{A}, \tau)$ be a NCPS.
Let $\mf{A}_1, \dots, \mf{A}_k$ be subalgebras having the same unit as $\mf{A}$. They are said to be \emph{free} if
the following holds:
for any $n \in \N$, any sequence $j_1, \dots, j_n \in [k]$, and any $a_i \in \mf{A}_{j_i}$ ($i=1, \dots, k$) with
\begin{align}
&\tau \left(a_i\right) = 0  \  (i=1, \dots, n),\\
&j_1 \neq j_2, j_2 \neq j_3, \dots, j_{n-1} \neq j_n,
\end{align}
the following holds true: 
 \begin{align}
	\tau \left( a_{j_1}  a_{j_2}  \dots a_{j_n} \right) = 0.
\end{align} 
Besides, elements in $\mf{A}$ are said to be \emph{free} iff the unital subalgebras that they generate are free.
\end{defn}

The example below is basically a reformulation of freeness, and follows from \cite{Voiculescu1991limit}.
\
\begin{example}
Let  $w_1,w_2, \dots, w_L \in \mf{A}$ and $d_1, \dots, d_L \in \mf{A}$. 
Then the families $(w_1,w_1^*), (w_2,w_2^*), \dots, (w_L, w_L^*), (d_1, \dots, d_L)$ are free if and only if the following $L+1$ unital subalgebras of $\mf{A}$ are free:
\begin{align}
& \{ P( w_1, w_1^*) \mid P \in \C \langle X, Y \rangle \},  \dots ,\{ P( w_L, w_L^*) \mid P \in \C \langle X, Y \rangle \}, \\
& \{  Q(d_1, \dots, d_L) \mid  Q \in \C \langle X_1, \dots, X_L \rangle \}.
\end{align}

\end{example}

Let us now introduce asymptotic freeness of random matrices with compact support limit spectral distributions. 
Since we consider a family of a finite number of random matrices, we restrict it to a finite index set.
Note that the finite index is not required for a general definition of freeness.
\begin{defn}[Asymptotic Freeness of Random Matrices]
    Consider a nonempty finite index set  $I$,
    a family $A_i(N)$
    of $N \times N$ random matrices where $N \in \N$.
    Given a partition $\{I_1, \dots, I_k\}$ of $I$,
	consider a sequence of $k$-tuples 
	\begin{align}
	   \left( A_i \left(N \right) \mid  i \in I_1 \right) , \dots ,   	   \left( A_i \left(N \right) \mid  i \in I_k \right). 
	 \end{align}
   It is then said to be
   almost surely
   \emph{asymptotically free} as $N \to \infty$   if the following two conditions are satisfied.
	\begin{enumerate}
		\item There exist a family $(a_i)_{i \in I}$ of elements in $\mf{A}$ such that the following $k$ tuple is free:
		\begin{align}
		    (a_i \mid i \in I_1)  , \dots, (a_i \mid i \in I_k ) .
		\end{align}
		\item 
		For every $P \in \C \langle X_1, \dots, X_{|I|} \rangle$, 
    	\begin{align}
			\lim_{N \to \infty}\tr  \left( P \left( A_1(N), \dots, A_{|I|}(N) \right)  \right) = \tau \left(  P \left(a_1 ,\dots, a_{|I|} \right)  \right),
		\end{align}
			almost surely, where $|I|$ is the number of elements of $I$.
	\end{enumerate}
\end{defn}

\subsection{Haar Distributed Orthogonal Random Matrices}
We introduce asymptotic freeness of Haar distributed orthogonal random matrices. 
\begin{prop}\label{prop:af}
Let $L,L' \in \N$. For any $N \in \N$, let $V_1(N), \dots, V_{L}(N)$ be independent  $\O{N}$ Haar random matrices, and $A_1(N), \dots, A_{L'}(N)$ be symmetric random matrices, which have the almost-sure-limit joint distribution.
Assume that all entries of $(V_\ell(N))_{\ell=1}^{L}$  are independent of that of 
    $(A_1(N), \dots, A_{L'}(N))$,
for each $N$.
Then the  families 
\begin{align}
 ( V_1(N), V_1(N)^\top ), \dots, (V_{L}(N), V_{L}(N)^\top), (A_1(N), \dots, A_{L'}(N)).   
\end{align}
are asymptotically free as $N \to \infty$.
\end{prop}
\begin{proof}
This is a particular case of \cite[Theorem~5.2]{collins2006integration}.
\end{proof}

The following proposition is a direct consequence of \cref{prop:af}.
\begin{prop}\label{prop:AVBV}
     For $N \in N$, let $A(N)$ and $B(N)$ be  $N \times N$ symmetric random matrices, and let $V(N)$ be a $N \times N$ Haar-distributed orthogonal random matrix.
    Assume that 
    \begin{enumerate}
        \item The random matrix $V(N)$ is  independent of $A(N), B(N)$ for every $N \in \N$.
        \item The spectral distribution of $A(N)$ (resp.\,$B(N)$)
        converges in distribution to a compactly supported probability measure 
        $\mu$ (resp.\,$\nu$), almost surely.
    \end{enumerate}
    Then  the following pair  is asymptotically free as $N \to \infty$,
    \begin{align}
     A(N), V(N)B(N)V(N)^\top ,
    \end{align}
    almost surely.
\end{prop}

\begin{proof}
Instead of proving that $A(N), V(N)B(N)V(N)^\top$ are asymptotically free, we will prove that
$U(N)A(N)U(N)^\top, U(N)V(N)B(N)V(N)^\top U(N)^\top$ for any orthogonal matrix $U(N)$, and
in particular, for an independent Haar distributed orthogonal matrix. 
This is equivalent because a global conjugation by $U(N)$ does not affect the joint distribution. 
In turn, since $U(N),U(N)V(N)$ has the same distribution as $U(N),V(N)$ thanks to the Haar 
property, it is enough to prove that 
$U(N)A(N)U(N)^\top, V(N)B(N)V(N)^\top$ is asymptotically free as as $N \to \infty$.
Let us replace $A(N)$ by $\tilde A(N)$ where $\tilde A(N)$ is diagonal, and 
has the same eigenvalues as $A(N)$, arranged in non-increasing order, and likewise, 
we construct $\tilde B(N)$ from $B(N)$. 
It is clear that
\begin{align}
U(N)A(N)U(N)^\top, V(N)B(N)V(N)^\top
\end{align}
and
\begin{align}
U(N)\tilde A(N)U(N)^\top, V(N)\tilde B(N)V(N)^\top    
\end{align}
have the same distribution. 
In addition, $\tilde B(N), \tilde A(N)$ have a joint distribution by construction, 
therefore we can apply \cref{prop:af}.

\end{proof}

Note that we do \textbf{not} require independence between $A(N)$ and $B(N)$ in \cref{prop:AVBV}.  
Here we recall the following result, which is a direct consequence of the translation invariance of Haar random matrices.
\begin{lemma}\label{lemma:UV}
Fix $N \in \N$.
Let $V_1, \dots, V_L$ be independent  $\O{N}$ Haar random matrices.
Let $T_1, \dots, T_L$ be $\O{N}$ valued random matrices.
Let $S_1, \dots, S_L$ be $\O{N}$ valued random matrices.
Let $A_1, \dots, A_L$ be $N \times N$ random matrices.
Assume that all entries of $(V_\ell)_{\ell=1}^L$  are independent of 
\begin{align}
    (T_1, \dots, T_L, S_1, \dots, S_L, A_1, \dots, A_L).
\end{align}
Then,
\begin{align}\label{align:TVS}
 (T_1 V_1 S_1, \dots, T_L V_L S_L, A_1, \dots, A_L)  \entries  ( V_1, \dots, V_L, A_1, \dots, A_L).   
\end{align}
\end{lemma}

\begin{proof}
For the readers' convenience, we include a proof. The characteristic function of  $(T_1 V_1 S_1, \dots, T_L V_L S_L, A_1, \dots, A_L)$ is given by 
\begin{align}\label{align:XTV+YA}
\E[\exp[-i \Tr [ \sum_{\ell=1}^L X_\ell ^\top T_\ell V_\ell S_\ell  + Y_\ell^\top A_\ell ]]],    
\end{align}
where $X_1, \dots, X_L \in M_N(\R)$ and $Y_1, \dots, Y_L \in M_N(\R)$.
By using conditional expectation,   \eqref{align:XTV+YA} is equal to 
\begin{align}\label{align:XTV+YA-c}
\E\left[\E\left[ \exp\left[-i \Tr \left( \sum_{\ell=1}^L X_\ell ^\top T_\ell V_\ell S_\ell \right) \mid T_\ell, S_\ell, A_\ell \left(\ell=1, \dots, L\right) \right] \exp\left[-i \Tr \left( Y_\ell^\top A_\ell \right) \right]\right]\right] .
\end{align}
By the property of the Haar measure and the independence,  the conditional expectation contained in \eqref{align:XTV+YA-c} is equal to \begin{align}
    \E\left[ \exp\left[-i \Tr \left( \sum_{\ell=1}^L X_\ell ^\top V_\ell \right) \right] \mid T_\ell, S_\ell, A_\ell \left(\ell=1, \dots, L\right) \right].
\end{align}
Thus the assertion holds.
\end{proof}

\subsection{Forward Propagation through MLP}

\subsubsection{Action of Haar Orthogonal Matrices}
Firstly we consider action of Haar orthogonal to a random vector with finite second moment.
For $N$-dimensional random vector $x=(x_1, \dots, x_n)$, we denote its empirical distribution by 
\begin{align}
\nu_x :=  \frac{1}{N}\sum_{n=1}^N \delta_{\post_n},
\end{align}
where $\delta_x$ is the delta probability measure at the point $x \in \R$.

\begin{lemma}\label{lem:hidden_convergence}
Let $(\Omega, \mathcal{F}, \mathbb{P})$ be a probability space and $x(N)$ be a $\R^N$ valued random variable for each $N\in\N$.
Assume that there exists $r > 0$ such that
\begin{align}
\sqrt{ \frac{1}{N} \sum_{n=1}^N \left( x(N)_n \right)^2 } \to r
\end{align}
as $N \to \infty$ almost surely. 
Let $O(N)$ be a Haar distributed $N$-dimensional orthogonal matrix. Set
\begin{align}
    h(N) =  O(N) x(N).
\end{align}
Furthermore we assume that  $x(N)$ and $O(N)$ are independent.
Then 
\begin{align}
 \nu_{h(N)} \implies \Normal(0, r^2)\end{align}
as $N \to \infty$ almost surely.
\end{lemma}

\begin{proof}
Let $e_1=(1, 0, \dots, 0) \in \R^N$. Then there is an orthogonal random matrix $U$ such that $x(N) = ||x(N)||_2 Ue_1$, where $|| \cdot ||_2$ is the Euclid norm.
Write $r(N) := ||x(N)||_2/\sqrt{N}$ and $u(N)$ be unit vector uniformly distributed on the unit sphere, independent of $r(N)$. 
Since $O(N)$ is a Haar orthogonal and since $O(N)$ and $U$ are independent, it holds that $O(N)U \sim^{dist.} O(N)$. Then
\begin{align}
    h(N) =  O(N)x(N) =  (||x(N)||_2/\sqrt{N} )( \sqrt{N}O(N)U e_1) \sim^{dist.} r(N) (\sqrt{N}u(N)).
\end{align}

Firstly, by the assumption, 
\begin{align}
    r(N) = \sqrt{ \frac{1}{N}\sum_{n=1}^N \post(N)_n^2 }\to r \text{ as \ $N \to \infty$,\  almost\ surely.}
\end{align}

Secondly, let $(Z_i)_{i=1}^\infty$ be i.i.d.\,standard Gaussian random variables.
Then 
\begin{align}
    u(N) \sim^{dist.}  \left( \frac{Z_n}{ \sqrt{ \sum_{n=1}^N Z_n^2 } } \right)_{n=1}^N.
\end{align}
For $k \in \N$,
\begin{align}
    m_k(\nu_{\sqrt{N}u(N)}) &= \frac{1}{N}\sum_{n=1}^{N}N^{k/2}u(N)_n^k = \frac{ N^{-1}\sum_{n=1}^N Z_n^k  }{[N^{-1} \sum_{n=1}^N Z_n^2 ]^{k/2} }\\  &\to \frac{m_k( \Normal(0,1) )}{m_2( \Normal(0,1) )^{k/2} } =  m_k( \Normal(0,1) ) \text{\ as \ $N \to \infty$, a.s.}
\end{align}
Now convergence in moments to  Gaussian distribution implies convergence in law.
Therefore, 
\begin{align}
    \nu_{\sqrt{N}u(N)}  \implies  \Normal(0, 1),
\end{align}
almost surely. This completes the proof.
\end{proof}

Note that we do \textbf{not} assume that entries of $x(N)$ are independent.

\begin{lemma}\label{lemma:function_convergence}
Let $g$ be a measurable function and set 
\begin{align}
N_g=\{x \in \R \mid g \text{\ is discontinuous at } x \}.    
\end{align}
Let $Z \sim \Normal(0,1)$.
Assume that $\mathbb{P}(Z \in N_g) = 0$.
Then under the setting of \cref{lem:hidden_convergence}, it holds that
\begin{align}
\nu_{g(h(N) ) } \implies g(Z)  
\end{align}
as $N \to \infty$ almost surely.
\end{lemma}
\begin{proof}
Let $F = \{\omega \in \Omega \mid \nu_{g(h(N)(\omega)) } \implies g(Z) \text{\ as \ } N \to \infty \}$.  By \cref{lem:hidden_convergence}, $P(F) = 0$.
Fix $\omega \in \Omega \setminus F$.
For $N \in \N$, let $X_N$ be a real random variable on the probability space with 
\begin{align}
    X_N \sim  \nu_{h(N)(\omega)}.
\end{align}
By the assumption, we have $\mathbb{P}( Z \in N_g ) = 0$. Then the continuous mapping theorem (see \cite[Theorem~3.2.4]{Durret2010probability}) implies that
\begin{align}
    g(X_N) \To g(Z).
\end{align}
Thus for any bounded continuous function $\psi$, 
\begin{align}
\int \psi(t) \nu_{g\circ h(N)(\omega)}( dt) =   \frac{1}{N}\sum_{i=n}^N \psi \circ g\left(\pre(n)\left(\omega\right) \right) = \E[ \psi \circ g(X_N) ] \to \E[\psi  \circ g(Z)].
\end{align}
Hence $\nu_{g(h(N)(\omega))} \implies g(Z)$. Since we took arbitrary $\omega \in \Omega \setminus F$ and $\mathbb{P}(\Omega \setminus F) = 1$, the assertion follows.
\end{proof}

\subsubsection{Convergence of Empirical Distribution}\label{ssec:emp}

Furthermore, for any measurable function $g$ on $\R$ and probability measure $\mu$, we denote by $g_*(\mu)$ the push-forward of $\mu$. That is, if a real random variable $X$ is distributed with $\mu$, then $g_*(\mu)$ is the distribution of $g(X)$.

\begin{prop}\label{prop:signal}
For all $\ell =1, \dots, L$, it holds that
\begin{enumerate}[ref=\text{Prop.\,\ref{prop:signal}\,(\arabic*)}]
	\item \label{enum:limit-of-h} $\nu_{\pre^\ell} \To \Normal(0, q_\ell),   $
    \item \label{enum:limit-of-x}$\nu_{\act^\ell(\pre^\ell)} \To \act^{\ell}_*(\Normal(0, q_\ell)),$  
    \item \label{enum:limit-of-d} $\nu_{(\act^\ell)^\prime(\pre^\ell)} \To (\act^\ell)^\prime_{*}(\Normal(0, q_\ell)),$
\end{enumerate}
as $N \to \infty$ almost surely.
\end{prop}
\begin{proof}
The proof is  by induction on $\ell$.
Let $\ell=1$.  Then $q_1 = \sigma_{w,1}^2 r^2 + \sigma_{b,1}^2$.
By \cref{lem:hidden_convergence},  \ref{enum:limit-of-h} follows.
Since $\act^1$ is continuous \ref{enum:limit-of-x} follows by \cref{lemma:function_convergence}. 
Since $(\act^1)^\prime$ is continuous almost everywhere by the assumption \eqref{enum:act_deriv_conti}, \ref{enum:limit-of-d} follows by  \cref{lem:hidden_convergence}.
Now we have $||\post^1||_2/\sqrt{N} = \sqrt{m_2(\nu_{\act^1(\pre^1)})} \To\sqrt{ m_2(\act^1_*(\Normal(0,q_1))) }= r_1$.
The same conclusion can be drawn for the rest of  induction.
\end{proof}

\begin{cor}\label{cor:limit-D}
 For each $\ell=1, \dots, L$, $D_\ell$ has the compactly supported limit spectral distribution $(\act^\ell)^\prime_{*}(\Normal(0, q_\ell))$ as $N \to \infty$.
\end{cor}
\begin{proof}
The assertion  follows directly from \cref{enum:limit-of-d} and \ref{enum:act_deriv_bounded}.
\end{proof}

\section{Key to Asymptotic Freeness}\label{sec:keys}

Here we introduce key lemmas to prove the asymptotic freeness. A key lemma is about an invariance of MLP, and the other one is about a property of cutting off matrices.

\subsection{Notations}\label{ssec:notations}
We prepare notations related to the change of basis to cut off  entries in $W_\ell$, which are correlated with $D_\ell$.

For $N \in \N$,  fix a standard complete orthonormal basis $(e_1, \dots, e_N)$ of $\R^N$. 
Firstly, set
    $\hat{n} = \min \{ n =1, \dots, N \mid \langle x^\ell, e_n  \rangle \neq 0\}$.
Since $\post^\ell$ is non-zero almost surely, $\hat{n}$ is defined almost surely.
Then the following family is a basis of $\R^N$: 
\begin{align}\label{align:family}
(e_1, \dots, e_{\hat{n} - 1}, e_{\hat{n} + 1}, \dots, e_N, \post^\ell/||\post^\ell||_2),
\end{align}
where $||\cdot ||_2$ is the Euclidian norm. Secondly, we apply the Gram-Schmidt orthogonalization to the basis \eqref{align:family} in reverse order, starting with $x^\ell/||x^\ell||_2$, to constrcut an orthonormal basis $(f_1, \dots, f_N=x^\ell/||x^\ell||_2)$.
Thirdly, let $Y_\ell$ be the orthogonal matrix determined by the following change of orthonormal basis:
\begin{align}\label{align:Y_defn}
    Y_\ell f_n = e_n  \ (n=1, \dots, N).
\end{align}
Then $Y_\ell$ satisfies the following conditions.
\begin{enumerate}
    \item $Y_\ell$ is $x^{\ell}$-measurable.
    \item $Y_\ell \post^\ell = ||\post^\ell||_2 e_N$.
\end{enumerate}
Lastly, let $V_0,  \dots, V_{L-1}$ be independent Haar distributed $ N-1 \times N-1$  orthogonal random matrices such that all entries of them are independent of that of   $(\post^0, W_1, \dots, W_L)$.
Set 
\begin{align}\label{align:practical_U}
    U_\ell = Y_\ell^\top \begin{pmatrix}
    V_\ell & 0 \\
    0 & 1
    \end{pmatrix} Y_\ell.
\end{align}
Then 
\begin{align}\label{align:practical_U_fixpoint}
    U_\ell \post^\ell = Y_\ell^\top ||\post^\ell||_2 e_N = \post^\ell.
\end{align}
Each $V_{\ell}$ is the $N-1 \times N-1$ random matrix which determines the action of $U_{\ell}$ on the orthogonal complement of $\R \post^{\ell}$.
Further, for any $\ell=0, \dots, L-1$, 
all entries of $(U_0, \dots, U_{\ell-1})$ are independent from that of  $(W_\ell, \dots, W_L)$ since each $U_{\ell}$ is $\mc{G}(x^\ell, V^\ell)$-measurable, where $\mc{G}(x^\ell, V^\ell)$ is the $\sigma$-algebra generated by $x^\ell$ and $V^\ell$.
We have completed the construction of the $U_\ell$.
\cref{fig:graphical-V} visualizes a dependency of the random variables that appeared in the above discussion.

\begin{figure}[h]
    \centering
    \begin{tikzpicture}
    \node [circ] (x0) at (-3,0) { $\boldsymbol{x}^{0}$};
    \node [box] (ws) at (-3,3) { $\boldsymbol{W}_{1}, \dots, \boldsymbol{W}_{\ell-1}$};

    \node [circ] (xl-1) at (0,0) { $\boldsymbol{x}^{\ell-1}$};
    \node [box] (ws2) at (4,3) { $\boldsymbol{W}_{\ell}, \dots, \boldsymbol{W}_{L}$};

    \node [box] (ul-1) at (0,3) { $\boldsymbol{U}_{\ell-1}$};
    \node [box] (up) at (1.5,3) { $\boldsymbol{V}_{\ell-1}$};


    \draw[every loop,
          auto=left,
          line width=0.2mm,
          >=latex,
          draw=black,
          fill=black]    
        (x0) edge (xl-1)
        (ws) edge (xl-1)
        (xl-1) edge (ul-1)
        (up) edge (ul-1)

        ;

\end{tikzpicture}
    \caption{A graphical model of random variables in a specific case using $V_\ell$ for $U_\ell$.  See \cref{fig:graphical} for the graph's drawing rule.  The node of $W_\ell, \dots, W_L$ is an isolated node in the graph.}
    \label{fig:graphical-V}
\end{figure}
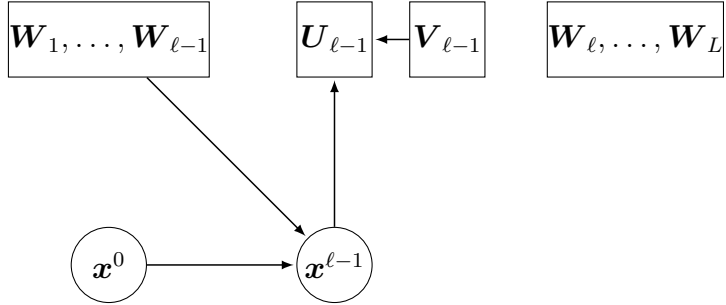

In addition, let $P(N)$ be the $N \times N$ diagonal matrix given by 
\begin{align}\label{align:P}
    P(N) =  \mathrm{diag}(1,1,\dots, 1, 0).
\end{align}
If there is no confusion, we omit the index $N$ and simply write it $P$. The matrix $P(N)$ is an orthogonal projection onto an $N-1$ dimenstional subspace.

\subsection{Invariance of MLP}\label{ssec:inv}
Since Haar random matrices' invariance leads to asymptotic freeness (\cref{prop:af}), it is essential to investigate the network's invariance. 
The following invariance is the key to the main theorem.
Note that the Haar property of $V_\ell$  is not necessary to construct $U_\ell$ in \cref{lemma:Ux}, but the property is used in the proof of \cref{thm:main}.

\begin{lemma}\label{lemma:Ux}
Under the setting of \cref{ssec:setting},
 let $U_\ell$ be arbitrary $\O{N}$ valued random matrix satisfying
\begin{align}\label{align:Ux-x}
    U_\ell \post^\ell = \post^\ell.
\end{align}
for each $\ell=0,1, \dots, L-1$.
Further assume that all entries of $(U_0, \dots, U_{\ell-1})$ are independent from that of $(W_\ell, \dots, W_L)$ for each $\ell=0,1,\dots, L-1$.
Then the following holds:
\begin{align}\label{align:joint-distributions}
(W_1U_0, \dots, W_LU_{L-1}, \pre^1, \dots, \pre^L)  \entries (W_1, \dots, W_L, \pre^1, \dots, \pre^L).
\end{align}
\end{lemma}



\begin{proof}[Proof of \cref{lemma:Ux}]
Let $U_0, \dots, U_{L-1}$ be arbitrary random matrices satisfing conditions in \cref{lemma:Ux}.
We prove the corresponding characteristic functions of the joint distributions in \eqref{align:joint-distributions} match.

%
Fix $T_1, \dots, T_L \in M_N(\R)$ and $\xi_1, \dots, \xi_L \in \R^N$. 
For each $\ell = 1, \dots, L$, define a map $\psi_\ell$ by 
\begin{align}
    \psi_\ell(x,W) = \exp\left[ -i  \Tr(T_\ell^\top W) -i \langle \xi_\ell, Wx\rangle \right],
\end{align}
where $W \in M_N(\R)$ and $x \in \R^N$.
Write 
\begin{align}
    \alpha_\ell &= \psi_\ell(\post_{\ell-1}, W_\ell), \label{align:alpha}\\
    \beta_\ell  &= \psi_\ell(\post_{\ell-1}, W_\ell U_{\ell-1}). \label{align:beta}
\end{align}
By \eqref{align:Ux-x} and by $W_\ell \post^{\ell-1} = \pre^\ell$, the values of characteristic functions of the joint distributions at the point $(T_1, \dots, T_L, \xi_1, \dots \xi_L)$ is given by 
$\E [\beta_1 \dots \beta_L]$ and $\E [\alpha_1 \dots \alpha_L]$, respectively.
%
%
Now we only need to show
\begin{align}\label{align:alpha-beta}
    \E [\beta_1 \dots \beta_L] = \E [\alpha_1 \dots \alpha_L].
\end{align}

Firstly, we claim that the following holds: for each $\ell=1, \dots, L$,
\begin{align}\label{align:ab-claim-2}
    \E[\beta_\ell \alpha_{\ell+1} \dots \alpha_L | \post^{\ell-1}] = \E[\alpha_\ell \alpha_{\ell+1} \dots \alpha_L | \post^{\ell-1}].
\end{align}
To show \eqref{align:ab-claim-2}, fix $\ell$ and write  for a random variable $x$, 
\begin{align}
    \mc{J}(x) = \E[\alpha_{\ell+1} \dots \alpha_L | x].
\end{align}
By the tower property of conditional expectations, we have
\begin{align}\label{align:tower-beta}
\E[\beta_\ell \alpha_{\ell+1} \dots \alpha_L | \post^{\ell-1}]=    \E[\beta_\ell \mc{J}(\post^{\ell}) | \post^{\ell-1}] &= \E[\E[\beta_\ell \mc{J}(\post^{\ell}) | \post^{\ell-1}, U_{\ell-1}] | \post^{\ell-1} ].
\end{align}
Let $\mu$ be the Haar measure. Then by the invariance of the Haar measure, we have
\begin{align}    
\E[\beta_\ell \mc{J}(\post^{\ell}) | \post^{\ell-1}, U_{\ell-1}]  
    &= \int \psi_\ell(\post^{\ell-1}, W U_{\ell-1}) \mc{J}(\phi_\ell(W U_{\ell-1} \post^{\ell-1})) \mu(dW)\\
    &=\int \psi_\ell(\post^{\ell-1}, W) \mc{J}(\phi_\ell(W \post^{\ell-1})) \mu(dW) \\
    & =\int \alpha_\ell \mc{J}(\post^\ell) \mu(dW)\\
    & = \E[\alpha_\ell \E[ \alpha_{\ell+1 } \dots \alpha_L | \post^\ell ] | \post^{\ell-1}]\\
    & = \E[\alpha_\ell \alpha_{\ell+1 } \dots \alpha_L | \post^{\ell-1}].
\end{align}
In particular, $\E[\beta_\ell \mc{J}(\post^{\ell}) | \post^{\ell-1}, U_{\ell-1}]$ is 
$x^{\ell-1}$-measurable.  By \eqref{align:tower-beta}, we have \eqref{align:ab-claim-2}.

Secondly, we claim that for each $\ell=2, \dots, L$,
\begin{align}\label{align:ab-claim-3}
    \E[ \beta_1 \dots \beta_{\ell-1}\beta_\ell \alpha_{\ell+1} \dots \alpha_L] = \E[\beta_1 \dots \beta_{\ell-1} \alpha_\ell \alpha_{\ell+1}\dots \alpha_L].
\end{align}
Denote by $\mc{G}$ the $\sigma$-algebra generated by $(x_0, W_1, \dots, W_{\ell-1}, U_0, \dots, U_{\ell-2})$.
By definition, $\beta_1, \dots, \beta_{\ell-1}$ are $\mc{G}$-measurable.
Therefore,
\begin{align}
    \E[\beta_1 \dots \beta_{\ell-1}\beta_\ell \alpha_{\ell+1} \dots \alpha_L] &= \E[\beta_1 \dots \beta_{\ell-1} \E[\beta_\ell \alpha_{\ell+1}\dots  \alpha_L | \mc{G} ] ] .
\end{align}
Now we have
\begin{align}
    \E[\beta_\ell \alpha_{\ell+1}\dots  \alpha_L | \mc{G}] &= \E[\beta_\ell \alpha_{\ell+1}\dots  \alpha_L | \post^{\ell-1} ],\\
    \E[\alpha_\ell \alpha_{\ell+1}\dots  \alpha_L | \mc{G}] &= \E[\alpha_\ell \alpha_{\ell+1}\dots  \alpha_L | \post^{\ell-1} ],
\end{align}
since the generators of $\mc{G}$ needed to determine $\beta_\ell, \alpha_\ell, \alpha_{\ell+1}, \dots \alpha_L$  are coupled into $\post^{\ell-1}$.
Therefore, by \eqref{align:ab-claim-2}, we have 
\begin{align}
 \E[\beta_1 \dots \beta_{\ell-1} \E[\beta_\ell \alpha_{\ell+1}\dots  \alpha_L | \mc{G} ] ]    &= \E[\beta_1 \dots \beta_{\ell-1} \E[\alpha_\ell \alpha_{\ell+1} \dots \alpha_L | \mc{G}]  ]\\
 &=\E[\beta_1 \dots \beta_{\ell-1} \alpha_\ell \alpha_{\ell+1} \dots \alpha_L ].
\end{align}
Therefore, we have proven \eqref{align:ab-claim-3}.

Lastly,  by applying \eqref{align:ab-claim-3} iteratively, we have
\begin{align}
    \E[ \beta_1 \beta_2 \dots \beta_L] = \E[ \beta_1 \alpha_2 \dots \alpha_L].
\end{align}
By \eqref{align:ab-claim-2}, 
\begin{align}
     \E[ \beta_1 \alpha_2 \dots \alpha_L] = \E[\E[\beta_1 \alpha_2 \dots \alpha_L | \post^0 ] ]=  \E[ \E[\alpha_1 \alpha_2 \dots \alpha_L | \post^0] ] = \E[\alpha_1 \alpha_2 \dots \alpha_L ].
\end{align}
We have completed the proof of \eqref{align:alpha-beta}.
\end{proof}

Here we visualize the dependency of the random variables in \cref{fig:graphical-char} in the case of the specific $(U_{\ell})_{\ell=0}^{L-1}$ in \eqref{align:practical_U} constructed with $(V_\ell)_{\ell=0}^{L-1}$. Note that we do not use the specific construction in the proof of \cref{lemma:Ux}.

\begin{figure}[h]
    \centering
    \begin{tikzpicture}
    \node [circ] (x0) at (-3,0) { $\boldsymbol{x}^{0}$};
    \node [box] (ws) at (-3,3) { $\boldsymbol{W}_{1}, \dots, \boldsymbol{W}_{\ell-1}$};

    \node [circ] (xl-1) at (0,0) { $\boldsymbol{x}^{\ell-1}$};
    \node [circ] (al) at (3,1.5) { $\boldsymbol{\alpha}_\ell$};
    \node [circ] (bl) at (1.5,1.5) { $\boldsymbol{\beta}_\ell$};
    \node [box] (wl) at (3,3) { $\boldsymbol{W}_\ell$};
    \node [box] (ul-1) at (0,3) { $\boldsymbol{U}_{\ell-1}$};
    \node [box] (up) at (1.5,3) { $\boldsymbol{V}_{\ell-1}$};


    \draw[every loop,
          auto=left,
          line width=0.2mm,
          >=latex,
          draw=black,
          fill=black]    
        (x0) edge (xl-1)
        (ws) edge (xl-1)
        (xl-1) edge (al)
        (xl-1) edge (bl)
        (xl-1) edge (ul-1)
        (up) edge (ul-1)

        (ul-1) edge (bl)

        (wl) edge (bl)
        (wl) edge (al)
        
        ;

\end{tikzpicture}
    \caption{A graphical model of random variables for computing characteristic functions in a specific case using $V_\ell$ for constructing $U_\ell$. See \eqref{align:alpha} and \eqref{align:beta} for the definition of $\alpha_\ell$ and $\beta_\ell$. See \cref{fig:graphical} for the graph's drawing rule. }
    \label{fig:graphical-char}
\end{figure}
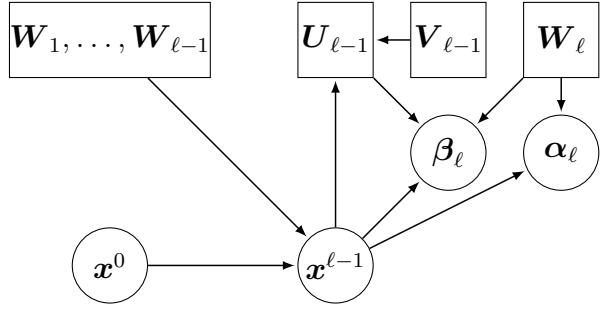

\subsection{Matrix Size Cutoff}
The invariance described in \cref{lemma:UV} fixes the vector  $\post^{\ell-1}$, and there are no restrictions on the remaining $N-1$ dimensional space $ P(N) \R^N$.
We call $P(N)AP(N)$ the catoff of any $N\times N$ matrix  $A$.
This section quantifies that cutting off the fixed space causes no significant effect when taking the large-dimensional limit.

For $p \geq 1$, we denote by $||X||_p$ the $L^p$-norm of $X \in M_N(\R)$ defined by 
\begin{align}
    || X ||_p= (\tr  |X|^p )^{1/p} = \left[\tr\left[\left(\sqrt{X^\top X}\right)^p\right]\right]^{1/p}.
\end{align}
Recall that the following non-commutative H{\"{o}}lder's inequality holds:
\begin{align}\label{align:Holder}
    || XY ||_r \leq || X ||_p || Y ||_q,
\end{align}
for any $r,p,q \geq 1$ with $1/r = 1/p + 1/q$.

\begin{lemma}\label{lemma:P}
Fix $n \in \N$. Let $X_1(N), \dots, X_n(N)$ be $N \times N$ random matrices for each $N \in \N$. Assume that there is a constant $C > 0$  satisfying almost surely
\begin{align}
    \sup_{N \in \N} \sup_{j=1, \dots, n}{||X_j(N)||_n } \leq  C.
\end{align}
Let $P(N)$ be the orthogonal projection defined in \eqref{align:P}.
Then we have almost surely
\begin{align}\label{align:PXP-X}
    |\tr[P(N)X_1(N)P(N) \dots P(N)X_n(N)P(N)]  - \tr[X_1(N) \dots X_n(N)] | \leq \frac{nC^n}{N^n}.
\end{align}
In particular, the left-hand side of \eqref{align:PXP-X} goes to $0$ as $N \to \infty$ almost surely.
\end{lemma}

\begin{proof}
We omit the index $N$ if there is no confusion.
Set
\begin{align}
    T = \sum_{j=0}^{n-1} PX_1 \cdots PX_{k-j-1}( P-1 )X_{n-j} X_{n-j+1} \cdots X_n.
\end{align}
Then the left-hand side of \eqref{align:PXP-X} is equal to $|\tr T|$.
By the H\"older's inequality \eqref{align:Holder}, 
\begin{align}
    |\tr T| \leq ||T||_1 \leq \sum_{j=0}^{n-1} ||P||_n^{n-j-1} ||X_1||_n \cdots  ||X_n||_n  || P -1 ||_n.
\end{align}
Now
\begin{align}
    || P -1 ||_n  = (\frac{1^n}{N})^{1/n} =  \frac{1}{N^{1/n}}.
\end{align}
Then by the assumption, we have $|\tr T| \leq nC^n/N^{1/n}$ almost surely.  
\end{proof}

By \cref{lemma:P}, the cutoff $P(N)XP(N)$ approximate $X$ in the sence of polynomials.
Next, we check that an orthogonal matrix approximates the cutoff of any orthogonal matrix. 
\begin{lemma}\label{lemma:PWP}
Let $N\in\N$ and $N \geq 2$. For any $\O{N}$ valued random matrix $W$, there is $W$-measurable $\O{N-1}$ valued random matrix $\grave{W}$ satisfying 
\begin{align}\label{align:pwp-u}
    || PWP - \begin{pmatrix} \grave{W} & 0 \\ 0 & 0 \end{pmatrix} ||_p \leq \frac{1}{(N-1)^{1/p}},
\end{align}
for any $p \in \N$ almost surely
\end{lemma}

\begin{proof}
Consider the singular value decomposition $(U_1, D, U_2)$ of $PWP$ in the $N-1$ dimensional subspace $P \R^N$, where $U_1, U_2$ belong to $\O{N-1}$, $D=\diag(\lambda_1, \dots, \lambda_{N-1})$, and $\lambda_1 \geq \dots \geq \lambda_{N-1}$ are singular values of $PWP$ except for the trivial singular value zero. Now
\begin{align}
 PWP =  \begin{pmatrix} U_1 D U_2 & 0 \\  0 & 0  \end{pmatrix}.
\end{align}
Set 
\begin{align}
\grave{W} = U_1 U_2.
\end{align}
Now $\grave{W}$ is $W$-measurable since $U_1$ and $U_2$ are determined by the singular value decomposition.
We claim that $\grave{W}$ is the desired random matrix.

We only need to show that $\Tr[ (1-D)^p ]\leq 1$, where $\Tr$ is the unnormalized trace.
Write 
\begin{align}
R = P - (PWP)^\top PWP =  PW^\top (1-P)WP.
\end{align}
Then $\rank R \leq 1$ and $\Tr R \leq ||WPW^\top || \Tr(1-P) \leq 1$.
Therefore, $R$'s nontrivial singular value belongs to $[0,1]$. We write it $\lambda$.
Then $(PWP)^\top PW =  P - R$ has nontrivial eigenvalue $1 - \lambda$ and eigenvalue $1$ of multiplicity $N-2$. Therefore,
\begin{align}
    D^2 = \diag(1, \dots, 1, 1- \lambda).
\end{align}
Thus $\Tr[(1-D)^p] = (1 - \sqrt{1-\lambda})^p \leq 1$. We have completed the proof.
\end{proof}

\section{Asymptotic Freeness of Layerwise Jacobians}\label{sec:main}
This section contains some of our main results. 
The first one is the most general form, but it relies on the existence of the limit joint moments of $(D_\ell)_\ell$.
The second one is required for the analysis of the dynamical isometry. 
The last one is needed for the analysis of the Fisher information matrix. 
The second and the third ones do not assume the existence of the limit joint moments of $(D_\ell)_\ell$.

We use the notations in \cref{ssec:notations}.
In the sequel, for each $\ell, N \in \N$, each $Y_\ell$ is the $\post^\ell$-measurable and $\O{N}$ valued random matrix  described in \eqref{align:Y_defn}. 
It is $\post^\ell$-measurable and satisfies
$Y_\ell \post^\ell = || \post^\ell ||_2 e_N$,
where  $e_N$ is the $N$-th vector of the standard basis of $\R^N$. 
Recall that $V_0, \dots, V_{L-1}$ are independent $\O{N-1}$ valued Haar random matrices such that all entries of them are independent of that of 
$(\post^0, W_1, \dots, W_L)$.
In addition, 
\begin{align}
    U_\ell = Y_\ell^\top \begin{pmatrix}
    V_\ell & 0 \\
    0 & 1
    \end{pmatrix} Y_\ell
\end{align}
and
$U_\ell \post^\ell = \post^\ell$.
Further, for any $\ell=0, \dots, \ell-1$, 
all entries of $(U_0, \dots, U_{\ell-1})$ are independent from that of  $(W_\ell, \dots, W_L)$.
Thus by \cref{lemma:Ux}, 
\begin{align}
    (W_1 U_0, \dots, W_L U_{L-1}, D_1, \dots, D_L) \entries (W_1, \dots, W_L, D_1, \dots, D_L).
\end{align}
In addition, for any $n \in \N$ and almost surely we have
\begin{align}\label{align:bdd-D}
        \max_{\ell=1, \dots, L} \sup_{n \in \N}||D_\ell||_n < \infty,
\end{align}
since each $D_\ell$ has the limit spectral distribution by \cref{cor:limit-D}.

We are now prepared to prove our main theorem.
\begin{thm}\label{thm:main}
Assume that $(D_1, \dots, D_L)$ has the limit joint distribution almost surely. Then the families $(W_1,W_1^\top), \dots ( W_L, W_L^\top)$, and $(D_1, \dots, D_L)$ are asymptotically free as $N \to \infty$ almost surely.
\end{thm}

\begin{proof}
Without loss of generality, we may assume that $\sigma_{w,1}, \dots, \sigma_{w,L}=1$.
Set 
\begin{align}
    Q_\ell = P W_\ell P Y_{\ell-1}^\top P \begin{pmatrix}V_{\ell-1} &  0\\ 0 & 1\end{pmatrix}P Y_{\ell-1} P.
\end{align}
for each $\ell=1, \dots,L$, where $P=P(N)$ is defined in \eqref{align:P}.
By \cref{lemma:P} and \eqref{align:bdd-D}, we only need to show the asymptotic freeness of the families
\begin{align}
(Q_\ell, Q_\ell^\top)_{\ell=1}^L, (PD_1P,\dots, PD_LP),
\end{align}
Now 
\begin{align}
    P \begin{pmatrix}V_{\ell-1} &  0\\ 0 & 1\end{pmatrix} P =  \begin{pmatrix}V_{\ell-1} &  0\\ 0 & 0\end{pmatrix}.
\end{align}
In addition,  let $\grave{D}_\ell$ be the $N-1 \times N-1$ matrix determined by
\begin{align}\label{align:breve-D}
    PD_\ell  P = \embedding{ \grave{D}_\ell}.
\end{align}
By \cref{lemma:PWP}, there are $\O{N-1}$ valued random matrices $\grave{W}_\ell$ and $\grave{Y}_{\ell-1}$ satisfying
\begin{align}\label{align:breve-WY}
    || P W_\ell P -  \embedding{ \grave{W}_\ell} ||_n &\leq \frac{1}{(N-1)^{1/n}},\\ 
    || P Y_\ell P -  \embedding{ \grave{Y}_\ell} ||_n &\leq \frac{1}{(N-1)^{1/n}},
\end{align}
for any $n \in \N$.
Therefore, we only need to show asymptotic freeness of the following $L+1$ families:
\begin{align}\label{align:free-breve}
    \left( \grave{W}_\ell \grave{Y}_{\ell-1}^\top  V_{\ell-1} \grave{Y}_{\ell-1}, \left(  \grave{W}_\ell \grave{Y}_{\ell-1}^\top  V_{\ell-1} \grave{Y}_{\ell-1}\right)^\top \right)_{\ell=1}^L, \left(\grave{D_1},\dots, \grave{D}_L\right).
\end{align}
Now all entries of Haar random matrices $(V_\ell)_\ell$ are independent of those of $(\grave{W}_\ell, \grave{Y}_{\ell-1}, \grave{D}_\ell )_\ell$. Thus by \cref{lemma:UV} and \cref{prop:af},  the asymptotic freeness of  \eqref{align:free-breve} holds as $N \to \infty$ almost surely. We have completed the proof.
\end{proof}

The following result is useful in the study of dynamical isometry and spectral analysis of Jacobian of DNNs.
It follows directly from \cref{thm:main} if we assume the existence of the limit joint moments of $(D_\ell)_{\ell=1}^L$. Note that the following result does \textbf{not} assume the existence of the limit joint moments.

\begin{prop}\label{prop:input-jacobian}
For each $\ell=1,\dots, L-1$, let $J_\ell$ be the Jacobian of $\ell$-th layer, that is, 
\begin{align}
    J_\ell = D_\ell W_\ell \dots D_1 W_1.
\end{align}
Then $J_\ell J_\ell^\top$ has the limit spectral distribution and the pair
\begin{align}\label{align:WJJW}
    W_{\ell+1}J_\ell J_\ell^\top W_{\ell+1}^\top , D_{\ell+1}^2
\end{align}
is asymptotically free as $N \to \infty$ almost surely.
\end{prop}
\begin{proof}
Without loss of generality, we may assume $\sigma_{w,1}, \dots, \sigma_{w,L}=1$.
We proceed by induction over $\ell$.

Let $\ell=1$. 
Then $J_1J_1^\top = D_1^2$ has the limit spectral distribution by  \cref{prop:signal}.
By \cref{lemma:P} and \eqref{align:bdd-D}, we only need to show that the asymptotically freeness of the pair
\begin{align}
    PW_2 P Y_{1}^\top P  \begin{pmatrix}V_1 & 0 \\
0 & 1 \end{pmatrix} P D_1^2 P \begin{pmatrix}V_1^\top & 0 \\
0 & 1 \end{pmatrix} P Y_1 P W_2^\top P, PD_2^2P,
\end{align}
By \cref{lemma:PWP}, there are $\grave{W}_2 \in \O{N-1}$ and $\grave{Y}_{1} \in \O{N-1}$ which  approximate $P W_2 P$ and $PY_1 P$ in the sence of \eqref{align:breve-WY}.
Let $\grave{D}_2$ be the $N-1 \times N-1$ random matrix given by \eqref{align:breve-D}.
Then, we only need to show the asymptotical freeness of the following pair:
\begin{align}
\grave{W}_2  \grave{Y}_{1}^\top V_1 \grave{D}_1^2 V_1^\top Y_1 \grave{W}_2^\top, \grave{D}_2^2.
\end{align}
By the independence and \cref{lemma:UV}, the asymptotic freeness holds almost surely.

Next, fix $\ell \in [1, L-1]$ and assume that the limit spectral distribution of $J_\ell J_\ell^\top$ exists and the asymptotic freeness holds for the $\ell$. 
Now
\begin{align}
    J_{\ell+1}J_{\ell+1}^\top = D_{\ell+1}W_{\ell+1}(J_\ell J_\ell^\top) W_{\ell+1}^\top D_{\ell+1}.
\end{align}
By the asymptotic freeness for the case $\ell$,      $J_{\ell+1}J_{\ell+1}^\top$ has the limit spctral distribution.
There exists $\grave{J}_\ell \in M_{N-1}(\R)$ so that 
\begin{align}
    P J_\ell P = \begin{pmatrix} \grave{J}_\ell & 0 \\ 0 & 0 \end{pmatrix}.
\end{align}
Then for the case of $\ell+1$, by the same argument as above, we only need to show the asymptotic freeness of 
\begin{align}
\grave{W}_{\ell+2}  \grave{Y}_{\ell+1}^\top  V_{\ell+1} \grave{J}_{\ell+1} \grave{J}^\top_{\ell+1} V_{\ell+1}^\top Y_{\ell+1} \grave{W}_{\ell+2}^\top, \grave{D}_{\ell+2}^2.
\end{align}
Now, all enties of $V_{\ell+1}$ are independent from those of  $(\grave{J}_{\ell+1}, \grave{W}_{\ell+1},  \grave{Y}_{\ell+1}, \grave{D}_{\ell+2})$.
By the independence and \cref{lemma:UV}, we only need to show the asymptotic freeness of 
\begin{align}
 V_{\ell+1} \grave{J}_{\ell+1} \grave{J}_{\ell+1}^\top V_{\ell+1}^\top,  \grave{D}_{\ell+2}^2.
\end{align}
The asymptotic freeness of the pair follows from  \cref{prop:AVBV}.
The assertion follows by  induction.
\end{proof}

Next, we treat a conditinal Fisher information matrix $H_L$ of the MLP. (See \cref{ssec:fisher}.)

\begin{prop}\label{prop:parameter-jacobian}
Define $H_\ell$ inductively by $H_1 = I_N$ and  
\begin{align}\label{align:hmltn_recurrence}
    \Hmltn_{\ell +1 } =   \hat{q}_{\ell}I + W_{\ell+1} D_{\ell}\Hmltn_{\ell}D_{\ell}W_{\ell+1}^\top, 
\end{align}
where $ \hat{q}_{\ell} = \sum_{j=1}^N (x^{\ell}_j)^2 / N $ and $\ell=1,\dots, L-1$.
Then  for each $ \ell = 1,2, \dots, L$,  $H_\ell$ has a limit spectral distribution and the pair
\begin{align}
    H_\ell, D_\ell
\end{align}
is asymptotically free as $N \to \infty$, almost surely.
\end{prop}

\begin{proof}
We proceed by induction over $\ell$.
The case $\ell=1$ is trivial.
Assume that the assertion holds for an $\ell \geq 1$ and consider the case $\ell+1$.
Then by \eqref{align:hmltn_recurrence} and the assumption of induction, $\Hmltn_{\ell+1}$ has the limit spectral distribution.
Let $\grave{H}_{\ell+1}$ be the $N-1 \times N-1$ matrix determined by
\begin{align}
    P H_{\ell} P = \embedding{\grave{H}}_{\ell}.
\end{align}
By the same arguments as above, we only need to prove the asymptotic freeness of the following pair:
\begin{align}
  \grave{W}_{\ell+1} \grave{Y}_\ell V_\ell \grave{Y}_\ell^\top \grave{D}_\ell \grave{H}_\ell \grave{ D}_\ell (W_{\ell+1}\grave{Y}_\ell V_\ell \grave{Y}_\ell^\top )^\top, \grave{D}_{\ell+1}.
\end{align}
By \cref{lemma:UV}, considering the joint distributions of all entries, we only need to show the asymptotic freeness of the following pair:
\begin{align}
      V_\ell \grave{D}_\ell \grave{H}_\ell \grave{ D}_\ell V_\ell^\top, \grave{D}_{\ell+1}.
\end{align}
By the assumption, $\grave{D}_\ell \grave{H}_\ell \grave{ D}_\ell$ has the limit spectral distribution. Then by \cref{prop:AVBV}, the assertion holds for $\ell+1$. The assertion follows 
by induction.
\end{proof}
\section{Application}\label{sec:appl}
Let $\nu_\ell$ be the limit spectral distribution of $D_\ell^2$ for each $\ell$.
We introduce applications of the main results.

\subsection{Jacobian and Dynamical Isometry}
Let $J$ be the Jacobian of the network with respect to the input vector.
In \cite{Saxe2014exact, Pennington2017resurrecting, Pennington2018emergence}, a DNN is said to achieve dynamical isometry if $J$ acts as a near isometry, up to some overall global $O(1)$ scaling, on a subspace of as high a dimension as possible.
Calling $\tilde H$ such a subspace, the $||(J^\top J)_{|\tilde H}-Id_{\tilde H}||_2=o(\sqrt{dim \tilde H})$.
Note that in \cite{Saxe2014exact, Pennington2017resurrecting, Pennington2018emergence}, a rigorous definition is not given, and that many variants of this definition are likely to be acceptable for the theory.
%
In their theory, they take firstly the wide limit $N \to \infty$.
To examine the dynamical isometry as the wide limit $N \to \infty$ and the deep limit $L \to \infty$, \cite{Pennington2017resurrecting, Pennington2018emergence, Hayase2020spectrum} consider $S$-transform of the spectral distribution. (See \cite{Voiculescu1987multiplication,rao2007multiplication} for the definition of $S$-transform).

Now 
\begin{align}
    J = J_L= D_L W_L \dots  D_1 W_1.
\end{align}
Recall that the existence of the limit spectral distribution of each $J_\ell$ $(\ell=1, \dots, L)$ is supported by \cref{prop:input-jacobian}.  
\begin{cor}\label{cor:s-trans}
Let $\xi_\ell$ be the limit spectral distribution as $N \to \infty$ of $J_\ell J_\ell^\top$. Then for each $\ell=1, \dots, L$, it holds that
\begin{align}\label{align:s-trans}
   S_{\xi_\ell}(z) =  \frac{1}{\sigma_{w,1}^2 \dots \sigma_{w,\ell}^2} S_{\nu_1}(z) \cdots S_{\nu_\ell}(z).    
\end{align}
\end{cor}
\begin{proof}
Consider the case $\ell=1$. Then $J_1 J_1^\top = D_1 W_1 W_1^\top D_1= \sigma_{w,1}^2 D_1^2$. 
Then  $S_{\xi_\ell}(z) = \sigma_{w,1}^{-2} S_{\nu_1}(z)$.

Assume that \eqref{align:s-trans} holds for an $\ell \geq 1$. Consider the case $\ell+1$.
By \cref{prop:input-jacobian},  $W_{\ell+1}^\top W_{\ell+1} = \sigma_{w, \ell+1}^2 I$ and the tracial condition, 
\begin{align}
    S_{\xi_{\ell+1}}(z) =   \frac{1}{\sigma_{w,\ell+1}^2} S_{\xi_{\ell}}(z) S_{\nu_{\ell+1}}(z).
\end{align}
The assertion holds by induction.
\end{proof}

\cref{cor:s-trans} is a resolution of an unproven result in \cite{Pennington2018emergence}, and it enables us to compute the deep limit $S_{\xi_L}(z)$ as $L \to \infty$.

\subsection{Fisher Information Matrix and Training Dynamics}\label{ssec:fisher}

We focus on the \emph{the Fisher information matrix} (FIM) for supervised learning with a mean squared error (MSE) loss \cite{Pascanu2013revisiting, Pennington2018spectrum, Karakida2019normalization}. 
Let us summarize its definition and basic properties. 
Given $x \in \R^N$ and parameters $\theta=(W_1, \dots, W_\ell)$, we consider a Gaussian probability model
\begin{align}\label{align:gaussian-model}
p_\theta(y | x)  = \frac{1}{\sqrt{2\pi}}\exp\left(- \mathcal{L}\left(f_\theta(x) -y\right) \right) \  (y \in \R^N).
\end{align}
Now, the normalized MSE loss $\mathcal{L}$ is given by $\mathcal{L}(u) = ||u||_2^2/2N$,
for $u \in \R^N$, and $||\cdot||_2$ is the Euclidean norm.
In addition, consider  a probability density function $p(x)$ and
a joint density $p_\theta(x,y) = p_\theta(y|x)p(x)$.
Then, the FIM is defined by
\begin{align}\label{align:original-fim}
    \mathcal{I}(\theta) = \int  [\nabla_\theta \log p_\theta(x,y )^\top \nabla_\theta \log p_\theta(x,y)] p_\theta(x,y)dxdy,
\end{align}
which is an $LN^2 \times LN^2$ matrix.
As it is known in information geometry \cite{Amari2016}, the FIM works as a degenerate metric on the parameter space: the Kullback-Leibler divergence between the statistical model and itself perturbed by an infinitesimal shift $d\theta$ is given by
$ D_{\mathrm{KL}}(p_\theta || p_{\theta+ d\theta}) = d\theta^\top   \mathcal{I}(\theta) d\theta.$ 
More intuitive understanding is that we can write the Hessian of the loss as 
\begin{align}
 \left(\pp{}{\theta}\right)^2\E_{x,y}[\mathcal{L}(f_\theta(x) - y)]   =    \mathcal{I}(\theta) + \E_{x,y}[ (f_\theta(x) -y )^\top    \left( \pp{}{\theta} \right)^2  f_\theta(x) ].
\end{align}
Hence the FIM also characterizes the local geometry of the loss surface around a global minimum with a zero training error. In addition, we regard $p(x)$ as an empirical distribution of input samples and then the FIM is usually referred to as the empirical FIM \cite{Kunstner2019limitations,Pennington2018spectrum,Karakida2019normalization}.

The conditional FIM  is used \cite{Hayase2020spectrum} for the analysis of training dynamics of DNNs achieving dynamical isometry. 
Now, we denote by $\mc{I}(\theta | x)$ the  \emph{conditional FIM} (or FIM per sample) given a single input $x$ defined by 
\begin{align}
    \mc{I}(\theta | x)  =  \int [\nabla_\theta \log p_\theta(y | x)^\top \nabla_\theta \log p_\theta(y | x)] p_\theta(y|x)dy.
\end{align}
Clearly, $\int \mc{I}(\theta | x) p(x)dx = \mathcal{I}(\theta)$.
Since $p_\theta(y|x)$ is Gaussian, we have
\begin{align}\label{align:cfim}
 \mc{I}(\theta | x) = \frac{1}{N}J_\theta^\top J_\theta.
\end{align}
Now, in order to ignore $\mathcal{I}(\theta|x)$'s trivial eigenvalue zero,
consider a dual of $\mc{I}(\theta |x)$ given by
\begin{align}\label{align:H-simplfied}
     \mathcal{J}(x,\theta) = \frac{1}{N} J_\theta J_\theta^\top,
\end{align}
which is an $N \times N$ matrix.
Except for trivial zero eigenvalues, $\mathcal{I}(\theta|x)$ and $\mathcal{J}(x,\theta)$ share the same eigenvalues as follows:
\begin{align}\label{align:dual-formula}
    \mu_{I(\theta|x)} = \frac{LN^2 -N }{LN^2} \delta_0 + \frac{1}{L} \mu_{\mathcal{J}(x, \theta)},
\end{align}
where $\mu_A$ is the spectral distribution for a matrix $A$.
Now, for simplicity, consider the case bias parameters are zero.
Then it holds that
\begin{align}
    \mathcal{J}(x,\theta) = D_L H_L D_L,
\end{align}
where
\begin{align}
    H_L &=  \sum_{\ell=1}^L \hat{q}_{\ell-1}\delta_{L \to \ell} \delta_{L \to \ell}^\top,\\
    \hat{q}_\ell &= ||x^\ell||_2^2/N,\\
    \delta_{L \to \ell}  &= \pp{\pre^L}{\pre^\ell}.
\end{align}
Since $\delta_{L \to \ell} = W_LD_{L-1} \delta_{L-1 \to \ell}$ $(\ell < L)$, it holds that
\begin{align}\label{align:reccurence_matrix}
\Hmltn_{\ell +1 } =   \hat{q}_{\ell}I + W_{\ell+1} D_{\ell}\Hmltn_{\ell}D_{\ell}W_{\ell+1}^\top, 
\end{align}
where $I$ is the identity matrix.

\begin{cor}
Let $\mu_{\ell}$ be the limit spectral distribution as $N \to \infty$ of $H_\ell$ ($\ell=1, \dots, L$).
Set $q_\ell = \lim_{N \to \infty}\hat{q}_{\ell}$. Then for each $\ell=1, \dots, L$ it holds that 
\begin{align}\label{align:H}
\mu_{\ell+1} =  (q_\ell +  \sigma_{\ell+1}^2 \cdot )_* (\mu_\ell \boxtimes \nu_\ell ),
\end{align}
where $f_*\mu$  is the pushforward of a measure $\mu$ by a  measurable map $f$.
\end{cor}

\begin{proof}
The assertion directly follows from \cref{prop:parameter-jacobian} and by induction.
\end{proof}

\cite{Hayase2020spectrum} uses the recursive equation \eqref{align:H} to compute the maximum value of the limit spectrum of $H_L$.
\section{Discussion}\label{sec:discussion}
We have proved the asymptotic freeness of MLPs with Haar orthogonal initialization by focusing on the invariance of the MLP.
\cite{Hanin2019products} shows the asymptotic freeness of MLP with Gaussian initialization and ReLU activation.  The proof relies on the observation that each ReLU's derivative can be replaced with independent Bernoulli from weight matrices.
On the contrary, our proof builds on the observation that weight matrices are replaced with independent random matrices from activations' Jacobians based on Haar orthogonal random matrices' invariance.
In addition, \cite{Yang2019scaling, Yang2020tensor} proves the asymptotic freeness of MLP with Gaussian initialization, which relies on Gaussianity.
Since our proof relies on the orthogonal invariance of weight matrices, our proof covers and generalizes the GOE case.

It is straightforward to extend our results including \cref{thm:main}
to MLPs with Haar unitary weights since the proof basely relies on the invariance of weight matrices (see \cref{lemma:Ux}) and the cut off (see \cref{lemma:PWP}).
We expect that our theorem can be extended to Haar permutation weights since Haar distributed random permutation matrices and independent random matrices are asymptotic free \cite{collins2006integration}.
Moreover, we expect that it is possible to extend the principal results and cover MLPs with orthogonal/unitary/permutation invariant random weights since each proof is based on the invariance of MLP.

The neural tangent kernel theory \cite{Jacot2018NeuralNetworks} describes the learning dynamics of DNNs when the dimension of the last layer is relatively smaller than the hidden layers. In our analysis, we do not consider such a case and instead consider the case where the last layer has the same order dimension as the hidden layers. 

\section*{Acknowledgement}
BC was supported by JSPS KAKENHI 17K18734 and 17H04823.
The research  of TH was supprted by JST JPM-JAX190N.
This work was supported by Japan-France Integrated action Program
(SAKURA), Grant number JPJSBP120203202.
\bibliographystyle{abbrv} 
\bibliography{references.bib}
\clearpage

\numberwithin{equation}{section}
\appendix

\section{Review on Probability Theory}

\begin{thm}[Continuous Mapping Theorem]\cite[Theorem~3.2.4]{Durret2010probability}\label{thm:cmt}
Let $g$ be a measurable function and 
\begin{align}
    N_g = \{  x \in \R \mid  g \text{ \ is discontinuous at \ } x\}.
\end{align}
If $X_n \To X_\infty$ and $P(X_\infty \in N_g) = 0$ then 
\begin{align}
    g(X_n) \To g(X_\infty).
\end{align}
If besides $g$ is bounded, then 
\begin{align}
     \E[g(X_n)] \to \E[g(X_\infty)].
\end{align}
\end{thm}

\begin{thm}\label{thm:LFCLT} (Lindeberg-Feller Central Limit Theorem for Triangular Array)\cite[Theorem~3.4.10]{Durret2010probability}
	
    For each $n$, let $X_{n,m}, 1 \leq m \leq n$, be independent random variables with $\E[X_{n,m}] = 0$. Suppose
    \begin{enumerate}
        \item for $n \in \N$, $\sum_{m=1}^n \E[X_{n,m}^2] \to \sigma^2 > 0$,
        \item for all $\eps > 0$, 
        $\lim_{n \to \infty}\sum_{m=1}^n \E[X_{n,m}^2 ; \abs{X_{n,m} } > \eps  ]= 0.$
    \end{enumerate}
    Then $S_n:= X_{n,1} + \dots X_{n,n} \To \Normal(0, \sigma^2)$ as $n \to \infty$.
\end{thm}

\begin{lemma}\label{lemma:clt-to-as}
	Let $X_{n,m}$ $(m=1, \dots,n )$ be i.i.d. random variables for each $n$. Assume that 
	\begin{align}
	X_{n,1} \To \mu, 
	\end{align}
	as $n \to \infty$
	for a probability measure  $\mu$.
	Then we have 
	\begin{align}
	\frac{1}{n} \sum_{m=1}^{n} \delta_{X_{n,m}} \To \mu, 
	\end{align}
	as $d \to \infty$ almost surely.
\end{lemma}

\begin{proof}
It suffices to show $\forall  f \in C_b(\R)$, 
	\begin{align}
	\frac{1}{n} \sum_{k=1}^n f(X_{n,m}) \to \int f(x) \mu(dx). 
	\end{align}
	as $n\to \infty $ almost surely.
	Set $\alpha_n := E[f(X_{n,m}) ]$ and $S_n := \sum_{k=1}^n (f(X_{n,m}) - \alpha_n )$.
	By assumption, we have 
	\begin{align}
	\lim_{n\to \infty}\alpha_{n}  = \int f(x)\mu(dx).
	\end{align}
	We claim that 
	\begin{align}
	\lim_{n \to \infty} \abs{S_n/ n } =  0  
	\end{align}
	almost surely.
	To show this, by Borel-Canteli's lemma, we only need to show that $\forall \epsilon >0$, 
	\begin{align}
	\sum_{n=1}^{\infty} P ( \abs{  S_n/n  } > \epsilon  ) < \infty.
	\end{align}
	Set $Y_n:=f(X_{n,m}) - \alpha_n$. Then 
	\begin{align}
	E[ S_n^4 ] =  E[Y_n ^4] - 3(n^2-n) E[Y_n^2]^2 < C n^2
	\end{align}
	for a constant $C >0$, because $f$ is bounded.
	Hence we have 
	\begin{align}
	P( \abs{  S_n^4/n^4 } > \epsilon^4 ) \leq \frac{1}{n^4 \epsilon^4}E[S_n^4] = \frac{C}{n^2 \epsilon^4}.  
	\end{align}
	Thus we have proved the assertion.
	
\end{proof}

\section{Review on Free Multiplicative Convolution}\label{sec:s-trans}
Let $(\mf{A}, \tau)$ be a NCPS. 
For $a \in \mf{A}$ with $\tau(a) \neq 0$, the \emph{S-transform} of $a$ is defined as  the formal power series
\begin{align}
	S_a(z) := \frac{1+z}{z}M_a^{<-1>}(z),
\end{align} 
where
\begin{align}
	M_a(z) := \sum_{n=1}^{\infty}\tau(a^n)z^n.
\end{align}
%
%
For example, given discrete distribution $\nu = \alpha\delta_0 + (1- \alpha)\delta_\gamma$ with $0\leq\alpha \leq 1$ and $\gamma>0$, we have
 $ S_\nu(z) = \gamma^{-1}(z+\alpha)^{-1}(z+1)$.

The relevance of S-transform in FPT is due to the following theorem of Voiculescu \cite{Voiculescu1987multiplication}:
  if $(a, b)$ is free with $\tau(a), \tau(b) \neq 0$ then 
\begin{align}
	S_{ab}(z) = S_{a}(z)S_b(z).
\end{align}
Further we assume that $a, b$ are self-adjoint and $b \geq 0$.
we define the multiplicative convolution 
\begin{align}
	\mu_a  \boxtimes \mu_b := \mu_{\sqrt{b}a \sqrt{b}}.   
\end{align}
Note that $ab$ will not self-adjoint.  But since $\tau$ is tracial, the distribution of  $ab$ is equal to  $\sqrt{b}a \sqrt{b}$.
Hence 
\begin{align}
 S_{\sqrt{b}a \sqrt{b} } = S_{ab} = S_a S_b.
\end{align}
Not that the S-transform can be defined in a more general situation \cite{rao2007multiplication}.

\end{document}